\documentclass{article}

\usepackage[nonatbib,preprint]{neurips_2020} 

\usepackage[utf8]{inputenc} 
\usepackage[T1]{fontenc}    
\usepackage{hyperref}       
\usepackage{url}            
\usepackage{booktabs}       
\usepackage{amsfonts}       
\usepackage{nicefrac}       
\usepackage{microtype}      

\usepackage{soul}
\usepackage{url}
\usepackage[small]{caption}
\usepackage{graphicx}
\usepackage{amsmath}
\usepackage{amsthm}
\usepackage{booktabs}
\usepackage{algorithm}
\usepackage{algorithmic}
\usepackage{caption}
\usepackage{subcaption}
\usepackage{xspace}
\usepackage{wrapfig}
\usepackage[usenames, dvipsnames]{color}
\usepackage[algo2e, noline, ruled]{algorithm2e}

\providecommand{\DontPrintSemicolon}{\dontprintsemicolon}
\DontPrintSemicolon
\newcommand{\btheta}{\boldsymbol\theta}
\newcommand{\cA}{\mathcal{A}}

\newcommand{\cD}{\mathcal{D}}

\newcommand{\cH}{\mathcal{H}}

\newcommand{\cN}{\mathcal{N}}

\newcommand{\cS}{\mathcal{S}}

\newcommand{\cZ}{\mathcal{Z}}

\newcommand{\defword}[1]{\textbf{\boldmath{#1}}}
\newcommand{\ie}{{\it i.e.}~}
\newcommand{\eg}{{\it e.g.}~}

\newtheorem{lemma}{Lemma}

\newcommand{\algname}{\textsc{ARMAC}\xspace}
\newcommand{\Politex}{\textsc{Politex}\xspace}

\newcommand{\E}{\mathbb{E}}
\newcommand{\1}{\mathbb{I}}

\definecolor{darkgreen}{RGB}{0,125,0}
\definecolor{amiiPink}{RGB}{241,97,119}
\definecolor{lime}{RGB}{255,200,0}

\newcounter{mlNoteCounter}

\newcounter{rkNoteCounter}

\newcounter{rmNoteCounter}

\newcounter{jpNoteCounter}

\newcounter{agNoteCounter}

\newcounter{dmNoteCounter}

\newcounter{msNoteCounter}

\newcounter{jsNoteCounter}

\newcounter{mgaNoteCounter}

\title{The Advantage Regret-Matching Actor-Critic}

\author{%
  Audrūnas Gruslys \\ DeepMind \\ \texttt{audrunas@google.com}
  \And
  Marc Lanctot \\ DeepMind
  \And
  Rémi Munos \\ DeepMind
  \And
  Finbarr Timbers \\ DeepMind
  \And
  Martin Schmid \\ DeepMind
  \And
  Julien Perolat \\ DeepMind
  \And
  Dustin Morrill \\ University of Alberta
  \And
  Vinicius Zambaldi \\ DeepMind
  \And
  Jean-Baptiste Lespiau \\ DeepMind
  \And
  John Schultz \\ NYU
  \And
  Mohammad Gheshlaghi Azar \\ DeepMind
  \And
  Michael Bowling \\ DeepMind and University of Alberta
  \And
  Karl Tuyls \\ DeepMind
}

\begin{document}

\maketitle

\begin{abstract}
Regret minimization has played a key role in online learning, equilibrium computation in games, and reinforcement learning (RL).
In this paper, we describe a general model-free RL method for no-regret learning based on repeated reconsideration of past behavior.
We propose a model-free RL algorithm, the Advantage Regret-Matching Actor-Critic (\algname):
rather than saving past state-action data, \algname saves a buffer
of {\it past policies}, replaying through them to reconstruct hindsight assessments of past behavior.
These retrospective value estimates are used to predict conditional advantages which, combined with regret matching, produces a new policy.
In particular, ARMAC learns from sampled trajectories in a centralized training setting, without requiring the application of importance sampling commonly 
used in Monte Carlo counterfactual regret (CFR) minimization; hence, it does not suffer from excessive variance in large environments.
In the single-agent setting, \algname shows an interesting form of exploration by keeping past policies intact.
In the multiagent setting, \algname in self-play approaches Nash equilibria on some partially-observable zero-sum benchmarks. We provide exploitability estimates in the significantly larger game of betting-abstracted no-limit Texas Hold'em.
\end{abstract}

\section{Introduction}

The notion of regret is a key concept in the design of many decision-making algorithms. 
Regret minimization drives most bandit algorithms, is often used as a metric for performance of reinforcement learning (RL) algorithms,
and for learning in games~\cite{Blum07}. When used in algorithm design, the common application is to accumulate values and/or regrets and derive
new policies based on these accumulated values. One particular approach, counterfactual regret (CFR) minimization~\cite{CFR}, has been the core algorithm behind 
super-human play in Computer Poker research~\cite{Bowling15Poker,Moravcik17DeepStack,Brown17Libratus,Brown19Pluribus}.

We investigate the problem of generalizing these regret minimization algorithms over large state spaces in the sequential setting using end-to-end function approximators, such as deep networks.
There have been several approaches that try to predict the regret, or otherwise, simulate the regret minimization: Regression CFR (RCFR)~\cite{Waugh15solving}, advantage regret 
minimization~\cite{Jin18ARM}, regret-based policy gradients~\cite{Srinivasan18RPG}, Deep Counterfactual Regret minimization~\cite{Brown18DeepCFR}, and 
Double Neural CFR~\cite{Li18DNCFR}. All of these approaches have focused either on the multiagent or single-agent problem
exclusively, some have used expert features, while others tree search to scale. Another common approach is based on fictitious play~\cite{Heinrich15FSP,Heinrich16NFSP,Lanctot17PSRO,Lockhart19ED}, a simple iterative self-play algorithm based on best response.
A common technique among several of these algorithms is to use reservoir sampling to maintain a
buffer that represents a uniform sample over past data, which is used to train a classifier representing the average policy. In Neural Fictitious Self-Play (NFSP), this produced competitive policies in limit Texas Hold'em~\cite{Heinrich16NFSP}, and in Deep CFR this method was shown to approach an approximate equilibrium in a large subgame of Hold'em poker.
A generalization of fictitious play, policy-space response oracles (PSRO)~\cite{Lanctot17PSRO}, stores past policies and a meta-distribution over them, replaying policies
against other policies, incrementally adding new best responses to the set, which can be seen as a population-based learning approach where the individuals are the policies and
the distribution is modified based on fitness. This approach only requires simulation of the policies and aggregating data; as a result,
it was able to scale to a very large real-time strategy game~\cite{Vinyals19AlphaStar}.
In this paper, we describe an approximate form of CFR in a training regime that we call {\it retrospective policy improvement}. 
Similar to PSRO, our method stores past policies. However, it does not store meta-distributions or reward tables, nor do the policies have to be approximate best responses, which can be costly to compute or learn.
Instead, the policies are snapshots of those used in the past, which are retrospectively replayed to predict a conditional advantage, which used in a regret matching algorithm produces the same policy as CFR would do.
In the single-agent setting, \algname is related to \Politex~\cite{Politex}, except that it is based on regret-matching~\cite{Hart00} and 
it predicts average quantities rather than explicitly summing over all the experts to obtain the policy. 
In the multiagent setting, it is a sample-based, model-free variant of RCFR with one important property: it
uses trajectory samples to estimate quantities {\it without requiring importance sampling} as in standard Monte Carlo CFR~\cite{Lanctot09mccfr}, hence it does not suffer from excessive variance in large environments.
This is achieved by using critics (value estimates) of past policies that are trained off-policy using standard policy evaluation techniques.
In particular, we introduce a novel training regime that estimates a conditional advantage $W_i(s,a)$, which is the cumulative counterfactual regret $R_i(s,a)$, scaled by factor $B(s)$ that depends on the information state $s$ only; hence, using regret-matching over this quantity yields the policy that CFR would compute when applying regret-matching to the same (unscaled) regret values.
By doing this entirely from sampled trajectories, the algorithm is model-free and can be done using any black-box simulator of the environment; hence, \algname inherits the scaling potential of PSRO without requiring a best-response training regime, driven instead by regret minimization.

\section{Background}

In this section, we describe the necessary terminology. Since we want to include the (partially-observable) multiagent case and we build on
algorithms from regret minimization
we use extensive-form games notations~\cite{ShohamLB09}.
A single-player game represents the single-agent case where histories are aggregated appropriately based on the Markov property.

A \defword{game} is a tuple $(\cN, \cA, \cS, \cH, \cZ, u, \tau)$, where $\cN = \{ 1, 2, \cdots, n\}$ is the set of players.
By convention we use $i \in \cN$ to refer to a player, and $-i$ for the other players $(\cN - \{ i \})$.
There is a special player $c$ called {\it chance} (or {\it nature}) that plays with a fixed stochastic strategy (chance's fixed strategy determines the transition function).
$\cA$ is a finite set of actions.
Every game starts in an initial state, and players sequentially take actions leading to histories of actions $h \in \cH$.
Terminal histories, $z \in \cZ \subset \cH$, are those which end the episode.
The utility function $u_i(z)$ denotes the player $i's$ return over episode $z$.
The set of states $\cS$ is a partition of $\cH$ where histories are grouped into \defword{information states}
$s = \{ h, h', \ldots \}$ such that the player to play at $s$, $\tau(s)$, cannot distinguish among the possible histories (world states)
due to private information only known by other players \footnote{Information state is the belief about the world that a given player can infer based on her limited observations and may correspond to many possible histories (world states)}.
Let $\Delta(X)$ represent all distributions over $X$: each player's (agent's) goal is to learn a policy $\pi_i : \cS_i  \rightarrow \Delta(\cA)$, where $\cS_i = \{s~|~ s \in \cS, \tau(s) = i\}$.
For some state $s$, we denote $\cA(s) \subseteq \cA$ as the legal actions at state $s$, and all valid state policies $\pi(s)$ assign probability 0 to illegal actions $a \not\in \cA(s)$.

Let $\pi$ denote a joint policy. Define the state-value $v_{\pi,i}(s)$ as the expected (undiscounted) return for player $i$
given that state $s$ is reached and all players follow $\pi$. Let $q_{\pi,i}$ be defined similarly except also
conditioned on player $\tau(s)$ taking action $a$ at $s$. Formally, $v_{\pi,i}(s) = \sum_{(h,z) \in \cZ(s)} \eta^\pi(h | s) \eta^\pi(h,z) u_i(z)$,
where $\cZ(s)$ are all terminal histories paired with their prefixes that pass through $s$, $\eta^\pi(h | s) = \frac{\eta^\pi(h)}{\eta^\pi(s)}$, where $\eta^\pi(s)=\sum_{h'\in s}\eta^\pi(h')$, and $\eta^\pi(h,z)$ is the product of probabilities of
each action taken by the players' policies along $h$ to $z$. The state-action values $q_{\pi,i}(s,a)$ are defined analogously.
Standard value-based RL algorithms estimate these
quantities for policy evaluation.
Regret minimization in zero-sum games uses a different
notion of value, the \defword{counterfactual value}: $v^c_{\pi,i}(s) = \sum_{(h,z) \in \cZ(s)} \eta^\pi_{-i}(h) \eta^\pi(h,z) u_i(z)$,
where $\eta^\pi_{-i}(h)$ is the product of opponents' policy probabilities along $h$.
We also write $\eta^\pi_{i}(h)$ the product of player $i$'s own probabilities along $h$.
Under the standard assumption of perfect recall, we have that for any $h,h'\in s$, $\eta^\pi_{i}(h)=\eta^\pi_{i}(h')$. Thus 
counterfactual values are formally related to the standard values~\cite{Srinivasan18RPG}:
$v_{\pi,i}(s) = \frac{ v^c_{\pi,i}(s) }{\beta_{-i}(\pi, s)}$,
where $\beta_{-i}(\pi, s) = \sum_{h \in s}\eta^{\pi}_{-i}(h)$. 
Also, $q^c_{\pi,i}(s,a)$ is defined similarly except over histories $(ha, z) \in \cZ(s)$,
where $ha$ is history $h$ concatenated with action $a$.

Counterfactual regret minimization (CFR) is a tabular policy iteration algorithm that has been at the core of
the advances in Poker AI~\cite{CFR}. On each iteration $t$, CFR computes
counterfactual values $q_{\pi,i}^c(s, a)$ and $v_{\pi,i}^c(s)$ for each state $s$ and action $a \in \cA(s)$ 
and the regret of {\it not} choosing action $a$ (or equivalently the advantage of choosing action $a$)
at state $s$, $r^t(s,a) = q_{\pi^t,i}^c(s, a) - v_{\pi^t,i}^c(s)$.
CFR tracks the cumulative regrets for each state and action, $R^T(s,a) = \sum_{t=1}^T r^t(s,a)$.
Define $(x)^+ = \max(0, x)$; regret-matching then updates the policy of each action $a \in \cA(s)$ as follows~\cite{Hart00}:
\begin{equation}
\label{eq:regret-matching}
\pi^{T+1}(s,a) = \textsc{NormalizedReLU}(R^T, s, a) =
\begin{cases}
\frac{R^{T,+}(s,a)} {\sum_{b \in \cA(s)} R^{T,+}(s,b)} & \text{if $\sum_{b \in \cA(s)} R^{T,+}(s,b) > 0$}\\
\frac{1}{|\cA(s)|}  & \text{otherwise}\\
\end{cases},
\end{equation}
In two-player zero-sum games, the mixture policy $\bar{\pi}^T$ converges to the set of Nash equilibria as $T \rightarrow \infty$.

Traditional (off-policy) Monte Carlo CFR (MCCFR) is a generic family of sampling variants~\cite{Lanctot09mccfr}.
In \defword{outcome sampling} MCCFR, a behavior policy $\mu_i$ is used by player $i$, while players $-i$
use $\pi_{-i}$, a trajectory $\rho \sim (\mu_i, \pi_{-i})$ is sampled, and the \defword{sampled
counterfactual value} is computed:
\begin{equation}
\label{eq:off-policy-sampled-cfv}
\tilde{q}_{\pi,i}^c(s,a~|~\rho) = \frac{1}{\eta^{(\mu_i, \pi_{-i})}_i(z)} \eta^{(\mu_i, \pi_{-i})}_i(ha,z) u_i(z),
\end{equation}
if $(s,a) \in \rho$, or $0$ otherwise. 
$\tilde{q}_{\pi,i}^c(s,a~|~\rho)$ is an unbiased estimator of $q^c_{\pi,i}(s,a)$~\cite[Lemma 1]{Lanctot09mccfr}. 

However, since these quantities are divided by $\eta^{(\mu, \pi_{-i})}_i(z)$, the product of player $i$'s probabilities, 
(i) there can be significant variance introduced by sampling, especially in problems involving long sequences of decisions, and (ii) the ranges of the $\tilde{v}_i^c$ 
can vary wildly (and unboundedly if the exploration policy is insufficiently mixed)
over iterations and states, which could make approximating the values in a general way particularly challenging~\cite{Waugh15solving}. 
Deep CFR and Double Neural CFR are successful large-scale implementations of CFR with function approximation, and they get around this variance issue by using external sampling or a robust sampling technique, both of which require a perfect game model and enumeration of the tree. This is unfeasible in very large environments or in the RL setting where full trajectories are generated from beginning to the end without having access to a generative model which could be used to generate transitions from any state.

\section{Retrospective Policy Improvement}

Retrospective policy improvement focuses on learning from past behavior, by storing complete descriptions of value functions and/or policies, deriving a new policy from them via aggregation, rather than e.g. greedily optimizing a current policy. 
Specifically, a retrospective agent finds a new policy by playing back through {\it past policies}, rather than maintaining a buffer of {\it past data}.
The general idea is common in learning partially-observable multiagent games, where algorithms are built upon regret minimization and learning from expert advice~\cite{Bubeck12}.
Regret minimization in self-play can approximate various forms of equilibria~\cite{Blum07}. In single-agent settings, they can also approximate optimal policies in Markov decision processes~\cite{Politex}, even when the reward function changes over time~\cite{Even-Dar05,kash2019combining}, and solve a general class of robust optimization problems~\cite{chen2012tractable}.
We also show some appealing properties in the single-agent case via some illustrative examples (Section~\ref{sec:single-agent}).


\subsection{The Advantage Regret-Matching Actor-Critic}

\algname is a model-free RL algorithm motivated by CFR.
Like algorithms in the CFR framework, \algname uses a centralized training setup and operates in epochs that correspond to CFR iterations.
Like RCFR, \algname uses function approximation to generate policies.
\algname was designed so that as the number of samples per epoch increases and the expressiveness of the function approximator approaches a lookup table, the generated sequence of policies approaches that of CFR.
Instead of accumulating cumulative regrets-- which is problematic for a neural network-- the algorithm learns a conditional advantage estimate $\bar{W}(s, a)$ by regression toward a history-dependent advantage $A(h,a)$, for $h\in s$, and uses it to derive the next set of joint policies that CFR would produce. Indeed we show that $\bar{W}(s, a)$ is an estimate of the cumulative regret $R(s,a)$ up to a multiplicative factor which is a function of the information state $s$ only, and thus cancels out during the regret-matching step.
\algname is a Monte Carlo algorithm in the same sense as MCCFR: value estimates are trained from full episodes. It uses off-policy learning for training the value estimates (\ie critics), which we show is sufficient to derive $\bar{W}$. However, contrary to MCCFR, it does {\em not} use importance sampling.
\algname is summarized in Algorithm~\ref{alg}.


\begin{wrapfigure}[35]{R}{0.6\textwidth}
\begin{minipage}{0.6\textwidth}
\vspace{-0.5cm}
\begin{algorithm2e}[H]
\SetKwInOut{Input}{input}\SetKwInOut{Output}{output}
\Input{initial set of parameters $\btheta^0$, number of players $n$}
\caption{Advantage Regret-Matching Actor-Critic\label{alg}}
$i \leftarrow 1$ \;
\For{epoch $t \in \{0, 1, 2, \cdots\}$}{
  reset $\cD \leftarrow \emptyset$ \;
  Let $\pi^t(s) = \textsc{NormalizedReLU}(\bar{W}_{\btheta^t}(s))$ (\ie Eq.~\ref{eq:regret-matching})\;
  Let $v_{\btheta^t}(h) = \sum_{a \in \cA(h)} \pi^t(h,a) q_{\btheta^t}(h,a)$ \;
  Let $\mu_i^t$ be a behavior policy for each player $i$ \;
  \For{episode $k \in \{1, \ldots, K_{act}\}$}{
        $i \leftarrow (i + 1)~\text{\bf mod}~n$ \;
        Sample $j \sim \textsc{Unif}(\{ 0, 1, \cdots, t-1\})$ \;
        Sample trajectory $\rho \sim (\mu_i, \pi_{-i}^j)$ \;
        let $d \leftarrow (i, j, \{ u_i(\rho) \}_{i \in \cN})$ \;
        \For{history $h \in \rho$ where player $i$ acts}{
            let $s$ be the state containing $h$ \;
            let $\vec{r} = \{ q_{\btheta^j}(h,a') - v_{\btheta^j}(h) \}_{a' \in \cA(s)}$ \;
            let $a$ be the action that was taken in $\rho$ \;f
            append $(h, s, a, \vec{r}, \pi^j(s))$ to $d$ \;
        }
        add $d$ to $\cD$\;
  }
  \For{learning step $k \in \{ 1, \ldots, K_{learn} \}$}{
    Sample a random episode/batch $d \sim \textsc{Unif}(\cD)$: \;
    \For{history and corresponding state $(h,s) \in d$}{
        Use Tree-Backup($\lambda$) to train the critic $q_{\btheta^t}(h,a)$ \;
        If $\tau(s) = i$: train $\bar{W}_{\btheta^t}$ to predict $A(h,a)$ \;
        If $\tau(s) \in -i$: train $\bar{\pi}_{\btheta^t}$ to predict $\pi^t(s)$ \;
    }
  }
  Save $\btheta^t$ for future retrospective replays \;
  $\btheta^{t+1} \leftarrow \btheta^t$ \;
}
\end{algorithm2e}
\end{minipage}
\end{wrapfigure}

\algname runs over multiple epochs $t$ and produces a joint policies $\pi^{t+1}$ at the end of each epoch.
Each epoch starts with an empty data set $\cD$ and simulates a variety of joint policies executing multiple training iterations of relevant function approximators. \algname trains several estimators which can be either heads on the same neural network, or separate neural networks. The first one estimate the history-action values $q_{\pi^t,i}(h, a) = \sum_{z\in\cZ(h,a)}\eta^{\pi^t}(h,z)u_i(z)$.
This estimator\footnote{In practice, rather than using $h$ as input to our approximators, we use a concatenation of all players' observations, i.e. an encoding of the {\it augmented information states} or {\it action-observation histories}~\cite{Burch14CFRD,Kovarik19FOG}. In some games this is sufficient to recover a full history. In others there is hidden state from all players, we can consider any chance event to be delayed until the first observation of its effects by any of the players in the game. Thus, the critics represent an expectation over those hidden outcomes. Since this does not affect the theoretical results, we choose this notation for simplicity. Importantly, \algname remains model-free: we never enumerate chance moves explicitly nor evaluate their probabilities which may be complex for many practical applications.}
can be trained on all previous data by using any off-policy policy evaluation algorithm  from experiences stored in replay memory (we use Tree-Backup$(\lambda)$~\cite{Precup00TB}).
If trained until zero error, this quantity would produce the same history value estimates as recursive CFR computes in its tree pass. 
Secondly, the algorithm also trains a state-action network $\bar{W}_i^t(s,a)$ that estimates the expected advantage $A_{\mu^t,i}(h,a)=q_{\mu^t,i}(h,a)-v_{\mu^t,i}(h)$ conditioned on $h\in s$ when following some mixture policy $\mu^t$ (which will be precisely defined in Section~\ref{sec:theory}). It happens that $\bar{W}_i^t(s,a)$ is an estimate of the cumulative regret $R^t(s,a)$ multiplied by a (non-negative) function which depends on the information state $s$ only, thus does not impact the policy improvement step by regret-matching (see Lemma~\ref{lemma:w}).
Once $\bar{W}_i^t(s, a)$ is trained, the next joint policy $\pi^{t+1}(s, a)$ can be produced by normalizing the positive part as in Eq.~\ref{eq:regret-matching}.
After each training epoch the joint policy $\pi^t$ is saved into a past policy reservoir, as it will have to be loaded and played during future epochs.
Lastly, an average policy head $\bar{\pi}^t$ is also trained via a classification loss to predict the policy $\pi^{t'}$ over all time steps $t' \le t$.
We explain its use in Section~\ref{sec:experiments}.

Using a history-based critic allows \algname to avoid using importance weight (IW) based off-policy correction as is the case in MCCFR, but at the cost of higher bias due to inaccuracies that the critic has. Using IW may be especially problematic for long games. For large games the critic will inevitably rely on generalization to produce history-value estimates.

To save memory, reservoir sampling with buffer of size of 1024 was used to prune past policies.

\subsection{Theoretical Properties}
\label{sec:theory}

Each epoch $t$ estimates
$q_{\pi^t,i}(h,a) = \sum_{z\in \cZ(h,a)}\eta^{\pi^t}(h,z)u_i(z)$
and value $v_{\pi^t, i}(h)=\sum_{a}\pi^t(h,a)q_{\pi^t,i}(h,a)$ for the current policies $(\pi^t)$. Let us write the advantages $A_{\pi^t,i}(h,a) = q_{\pi^t,i}(h,a) - v_{\pi^t, i}(h)$. Notice that we learn functions of the history $h$ and not state $s$. 

At epoch $T$, in order to deduce the next policy, $\pi^{T+1}$, CFR applies regret-matching using the cumulative counterfactual
regret $R_i^T(s,a)$. As already discussed, directly estimating $R_i^T$ using sampling suffers from high variance due to the inverse probability $\eta^{(\mu, \pi_{-i})}_i(z)$ in  \eqref{eq:off-policy-sampled-cfv}. Instead, \algname trains a network
$\bar{W}_i^T(s,a)$ that estimates a conditional advantage along
trajectories generated in the following way: For player $i$ we select a behavior policy $\mu_i^T$ providing a good state-space coverage, \eg a mixture of past policies $(\pi^t_i)_{t \leq T}$, with some added exploration (Section \ref{sec:ARMAC-bandit} provides more details). For the other players $-i$, for every trajectory, we choose one of the previous opponent policies $\pi_{-i}^j$ played at some epoch $j$ chosen uniformly at random from $\{ 1, 2, \cdots T \}$. Thus at epoch $T$, several trajectories $\rho^j$ are generated by following policy $(\mu^T_i,\pi^j_{-i})$, where 
$j\sim {\cal U}(\{ 1, 2, \cdots T \})$.

Then at each step $(h,a)$ along these trajectory $\rho^j$, the neural network estimate $\bar{W}_i^T(s,a)$ (where $s \ni h$) is trained to predict the advantage $A_{\pi^j,i}(h,a)$ using the empirical $\ell_2$ loss:  $\hat{\cal L}= \big[\bar{W}_i^T(s,a) - A_{\pi^j,i}(h,a)\big]^2$. 
Thus the corresponding average loss is 
$${\cal L}=\frac 1T\sum_{j=1}^T\E_{\rho^j\sim (\mu_i^T,\pi_{-i}^j)}\big[ \hat {\cal L} \big]= \frac 1T\sum_{j=1}^T \sum_{s\in \cS_i}\sum_{h\in s} \eta^{(\mu_i^T,\pi_{-i}^j)}(h) \mu_i^T(s,a) \big[\bar{W}_i^T(s,a) - A_{\pi^j,i}(h,a)\big]^2.$$
Thus if the network has sufficient capacity, it will minimize this average loss, thus our estimate $\bar{W}_i^T(s,a)$ will converge (when the number of trajectories goes to infinity) in each state-action pair $(s,a)$, such that the reach probability $\frac 1T\sum_t \eta^{(\mu_i^T,\pi_{-i}^t)}(s) \mu_i^T(s,a)>0$, to the conditional expectation
\begin{align}\label{eq:W.def}
W^T_i(s,a)= \sum_{h\in s} \frac{\frac 1T\sum_{j=1}^T \eta^{(\mu_i^T,\pi_{-i}^j)}(h)A_{\pi^k, i}(h,a)}{\frac 1T\sum_{j=1}^T \eta^{(\mu_i^T,\pi_{-i}^t)}(s)} \underbrace{=}_{\textrm{perfect recall}}\sum_{h\in s} \frac{\frac 1T\sum_{j=1}^T \eta^{\pi^j}_{-i}(h)A_{\pi^k, i}(h,a)}{\frac 1T\sum_{j=1}^T \eta^{\pi^t}_{-i}(s)}
\end{align}
Notice that $W_i^T$ does not depend on the exploratory policy $\mu_i^T$ for player $i$ chosen in round $T$. After several trajectories $\rho^j$ our network $\bar{W}^T_i$ provides us with a good approximation of the $W^T_i$ values and we use it in a regret matching update to define the next policy, $\pi_i^{T+1}(s) = \textsc{NormalizedReLU}(\bar{W}_i^T)$, \ie Equation~\ref{eq:regret-matching}.
The following lemma (proved in Appendix \ref{sec:proof}) shows that if $\bar{W}^T_i(s,a)$ is sufficiently close to the $W^T_i(s,a)$ values, then this is equivalent to CFR, i.e., doing regret-matching using the cumulative counterfactual regret $R^T$.

\begin{lemma}
\label{lemma:w}
The policy defined by $\textsc{NormalizedReLU}(W_i^T)$ is the same as the one produced by CFR when regret matching is employed as the information-state learner:
\begin{equation}\label{eq:same-policy}
\pi^{T+1}_i(s,a) = \frac{R_i^{T,+}(s,a)}{\sum_b R_i^{T,+}(s,b)} = \frac{W_i^{T,+}(s,a)}{\sum_b W_i^{T,+}(s,b)}.
\end{equation}
\end{lemma}

The $\bar{W}^T(s,a)$ estimate the expected advantages $\frac 1T\sum_{j=1}^T A_{\pi^j}(h,a)$ conditioned on $h\in s$. Thus \algname does not suffer from the variance of estimating the cumulative regret $R^T(s,a)$, and in the case of infinite capacity, we can prove that, from any $(s,a)$, the estimate $\bar{W}^T(s,a)$ is unbiased as soon as the  $(s,a)$ has been sampled at least once (see Lemma~\ref{lemma:unbiased} in the Appendix).

\subsection{Adaptive Policy Selection}
\label{sec:ARMAC-bandit}

\begin{figure}[thbp]
\centering
\begin{subfigure}{0.32\textwidth}
  \centering
  \includegraphics[width=1.0\linewidth]{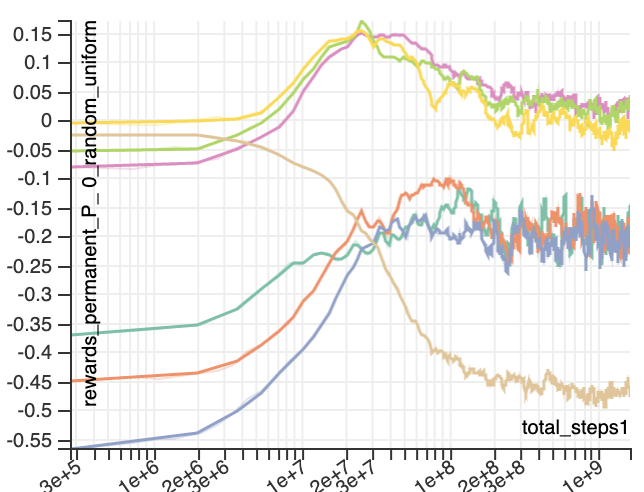}
  \caption{Goofspiel}
\end{subfigure}%
\begin{subfigure}{0.32\textwidth}
  \centering
  \includegraphics[width=1.0\linewidth]{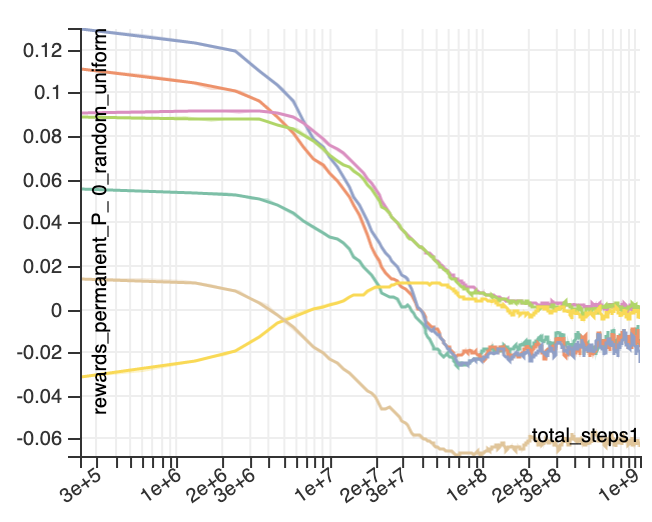}
  \caption{Leduc Poker}
\end{subfigure}%
\begin{subfigure}{0.32\textwidth}
  \centering
  \includegraphics[width=1.0\linewidth]{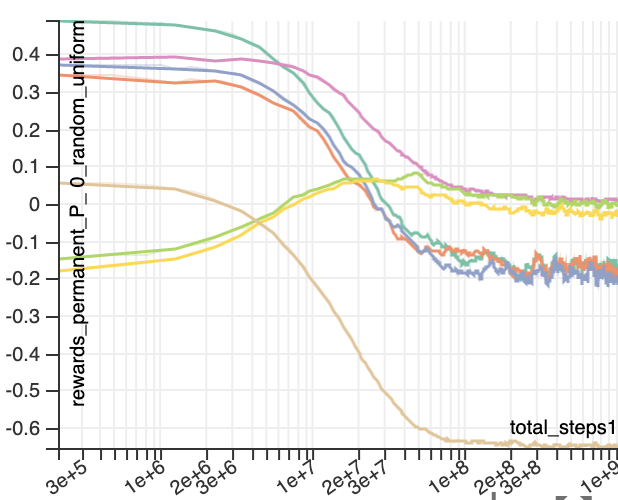}
  \caption{Liars Dice}

\end{subfigure}
\caption{An average reward a given policy modulation scores against opponent $\bar{\pi}^t$ as a function of time (measured in acting steps). The brown curve is a random uniform policy (i). 
Cyan, orange and blue is (ii) with $\epsilon \in {0.0, 0.01, 0.05}$ respectively.
Pink, green and yellow is (iii) with $\epsilon \in {0.0, 0.01, 0.05}$ respectively. }
\label{fig:ARMAC-bandit}
\end{figure}

\algname dynamically switches between what policy to use based on estimated returns. 
For every $t$ there is a pool of candidate policies, all based on the following four policies:
(i) random uniform policy.
(ii) several policies defined by applying Eq~\ref{eq:regret-matching} over the current epoch's regret only ($q_{\btheta^t}(h,a) - v_{\btheta^t}(h)$), with different levels of random uniform exploration: $\epsilon \in {0.0, 0.01, 0.05}$ . 
(iii) several policies defined by the mean regret, $\pi^t$ as stated in Algorithm~\ref{alg}, also with the same level of exploration.
(iv) the average policy $\bar{\pi}^t$ trained via classification.
The purpose of generating experiences using those policies is to facilitate the problem of exploration and to help produce meaningful data at initial stages of learning before average regrets are learnt.
Each epoch, the candidate policies are ranked by cumulative return against an opponent playing $\bar{\pi}_{\btheta^t}$. The one producing highest rewards is used half of the times. When sub-optimal policies are run for players $-i$, they are not used to train mean regrets for player $i$, but can be used to train the critic. Typically (but not always), (ii) produces the best policy initially and allows to bootstrap the learning process with the best data (Fig. \ref{fig:ARMAC-bandit}). In later stages of learning, (iii) with the lowest level of $\epsilon$ starts providing better policies and gets consistently picked over other policies. The more complex the game is, the longer it takes for (iii) to take over (ii).

Exploratory policy $\mu_i^T$ is constructed by taking the most recent neural network with $50\%$ probability or otherwise sampling one of the past neural networks from replay memory uniformly and modulating it by the above described method. Such sampling method provides both on-policy trajectories for learning $q_{\btheta^j}^T(h,a)$ while also making sure that previously visited states get revisited.

\subsection{Single-Agent Environments}
\label{sec:single-agent}

Despite \algname being based on commonly-used multiagent algorithms, it has properties that may be desirable in the single-agent setting.
First, similar to policy gradient algorithms in the common ``short corridor example''~\cite[Example 13.1]{SuttonBarto18}, stochastic policies are representable by definition, since they are normalized positive mean regrets over the actions. This could have a practical effect that entropy bonuses typically have in policy gradient methods, but rather than simply adding arbitrary entropy, the relative regret over the set of past policies is taken into account.

\begin{figure}
\centering
\begin{tabular}{cc}
  \includegraphics[width=0.35\linewidth]{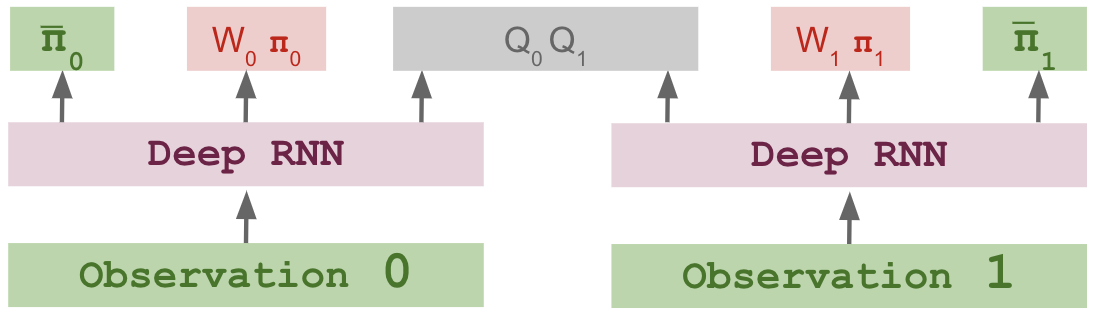} & 
  \includegraphics[width=0.4\linewidth]{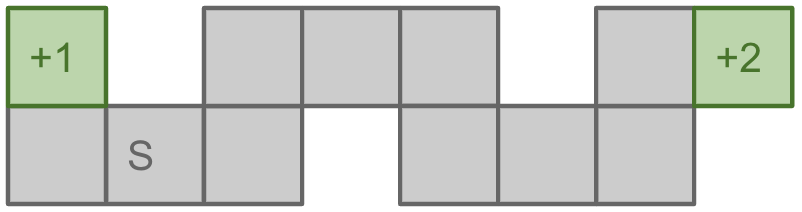} \\
  (a) & (b) \\
\end{tabular}
\caption{The (a) Multi-headed network architecture, and (b) Exploration example.}
\label{fig:netarch-gridworld}
\end{figure}

Second, a retrospective agent uses a form of {\it directed exploration} of different exploration policies~\cite{Badia19}. Here, this is achieved by the simulation $(\mu_i^T, \pi_{-i}^t)$, which could be desirable whenever there is overlapping structure in successive tasks. $\mu_i^T$ here is an exploratory policy, which consists of a mixture of all past policies (plus random uniform) played further modulated with different amounts of random uniform exploration (more details are given in Section \ref{sec:ARMAC-bandit}). Consider a gridworld illustrated in Fig. \ref{fig:netarch-gridworld}(b). Green squares illustrate positions where the agent $i$ gets a reward and the game terminates. Most of RL algorithms would find the reward of $+1$ first as it is the closest to the origin $S$. Once this reward is found, a policy would quickly learn to approach it, and finding reward $+2$ would be problematic. \algname, in the meantime, would keep re-running old policies, some of which would pre-date finding reward $+1$, and thus would have a reasonable chance of finding $+2$ by random exploration. This behaviour may also be useful if instead of terminating the game, reaching one of those two rewards would start next levels, both of which would have to be explored.

These properties are not necessarily specific to \algname. For example, \Politex (another retrospective policy improvement algorithm~\cite{Politex}) has similar properties by keeping its past approximators intact.
Like \Politex, we show an initial investigation of \algname in Atari in Appendix~\ref{sec:single-agent-atari}. Average strategy sampling MCCFR~\cite{gibson2012efficient} also uses exploration policies that are a mixture of previous policies and uniform random to improve performance over external and outcome sampling variants. However, this exact sampling method cannot be used directly in \algname as it requires a model of the game.

\subsection{Network architecture}

\algname can be used with both feed-forward (FF) and recurrent neural networks (RNN) (Fig. \ref{fig:netarch-gridworld}(a)). For small games where information states can be easily represented, FF networks were used. For larger games, where consuming observations rather than information states is more natural, RNNs were used. The critic was evaluated by feeding both player observations, but to the respective sides of the network. Policies only get their own observations to the respective side of the network. Both sides of the network in Fig. \ref{fig:netarch-gridworld}(a) share weights. More details can be found in Appendix in Section \ref{sec:net_arch_2}.

\section{Empirical Evaluation}
\label{sec:experiments}

For partially-observable multiagent environments, we investigate Imperfect Information (II-) Goofspiel, Liar's Dice, and Leduc Poker and betting-abstracted no-limit Texas Hold'em poker (in Section~\ref{sec:poker}).
Goofspiel is a bidding card game where players spend bid cards collect points from a deck of point cards.
Liar's dice is a 1-die versus 1-die variant of the popular game where players alternate bidding on the dice values.
Leduc poker is a two-round poker game with a 6-card deck, fixed bet amounts, and a limit on betting.
Longer descriptions of each games can be found in~\cite{Lockhart19ED}.
We use OpenSpiel~\cite{LanctotEtAl2019OpenSpiel} implementations with default parameters
for Liar's Dice and Leduc poker, 
and a 5-card deck and descending points order for II-Goofspiel.
To show empirical convergence, we use NashConv, the sum over each player's incentive to deviate to their best response
unilaterally~\cite{Lanctot17PSRO}, which can be interpreted as an empirical distance from Nash
equilibrium (reaching Nash at 0).


\begin{figure}[thbp]
\centering
\begin{subfigure}{0.32\textwidth}
  \centering
  \includegraphics[width=1.0\linewidth]{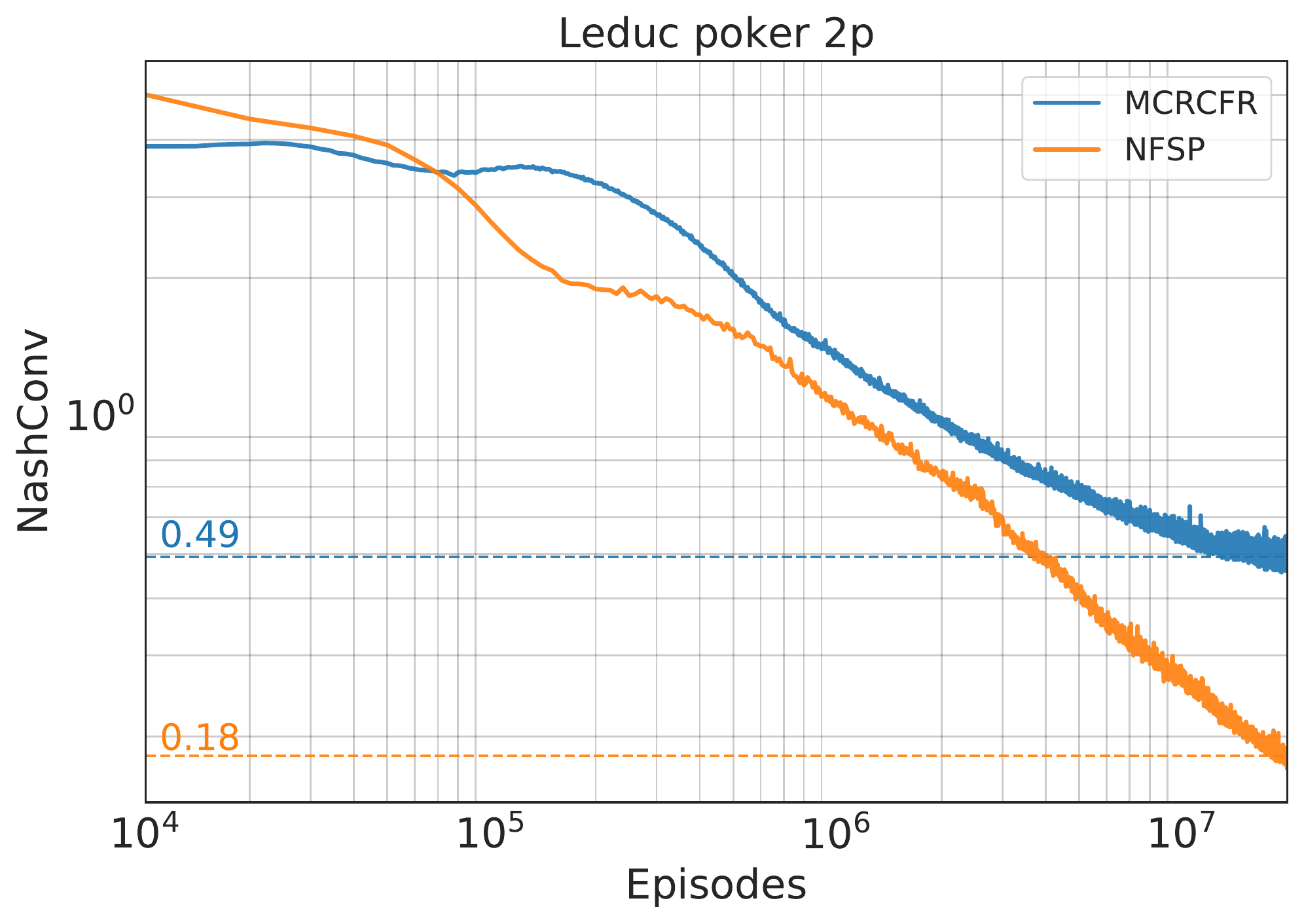}
  \caption{Leduc baselines.}
\end{subfigure}%
\begin{subfigure}{0.31\textwidth}
  \centering
  \includegraphics[width=1.0\linewidth]{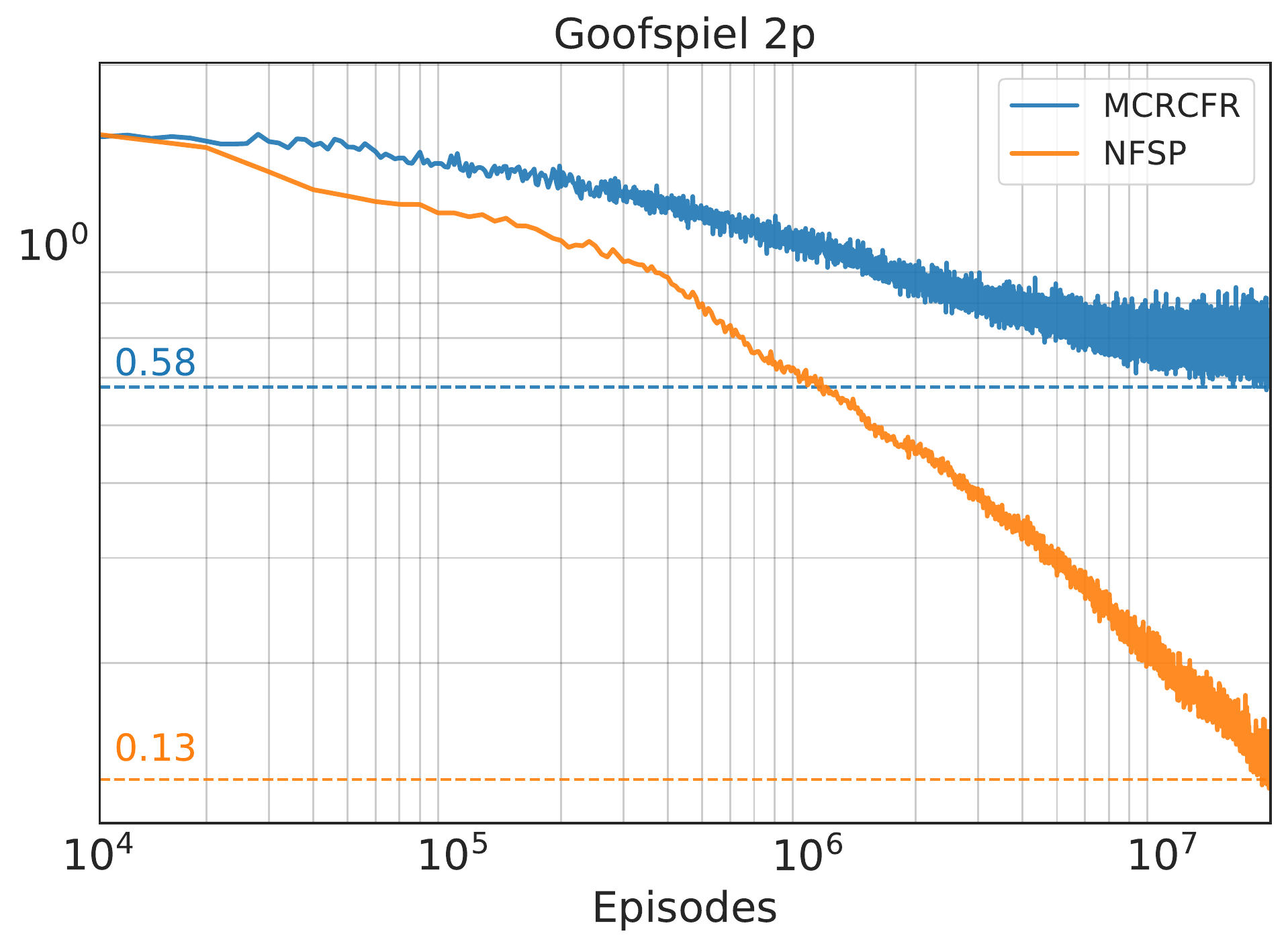}
  \caption{Goofspiel baselines.}
\end{subfigure}%
\begin{subfigure}{0.32\textwidth}
  \centering
  \includegraphics[width=1.0\linewidth]{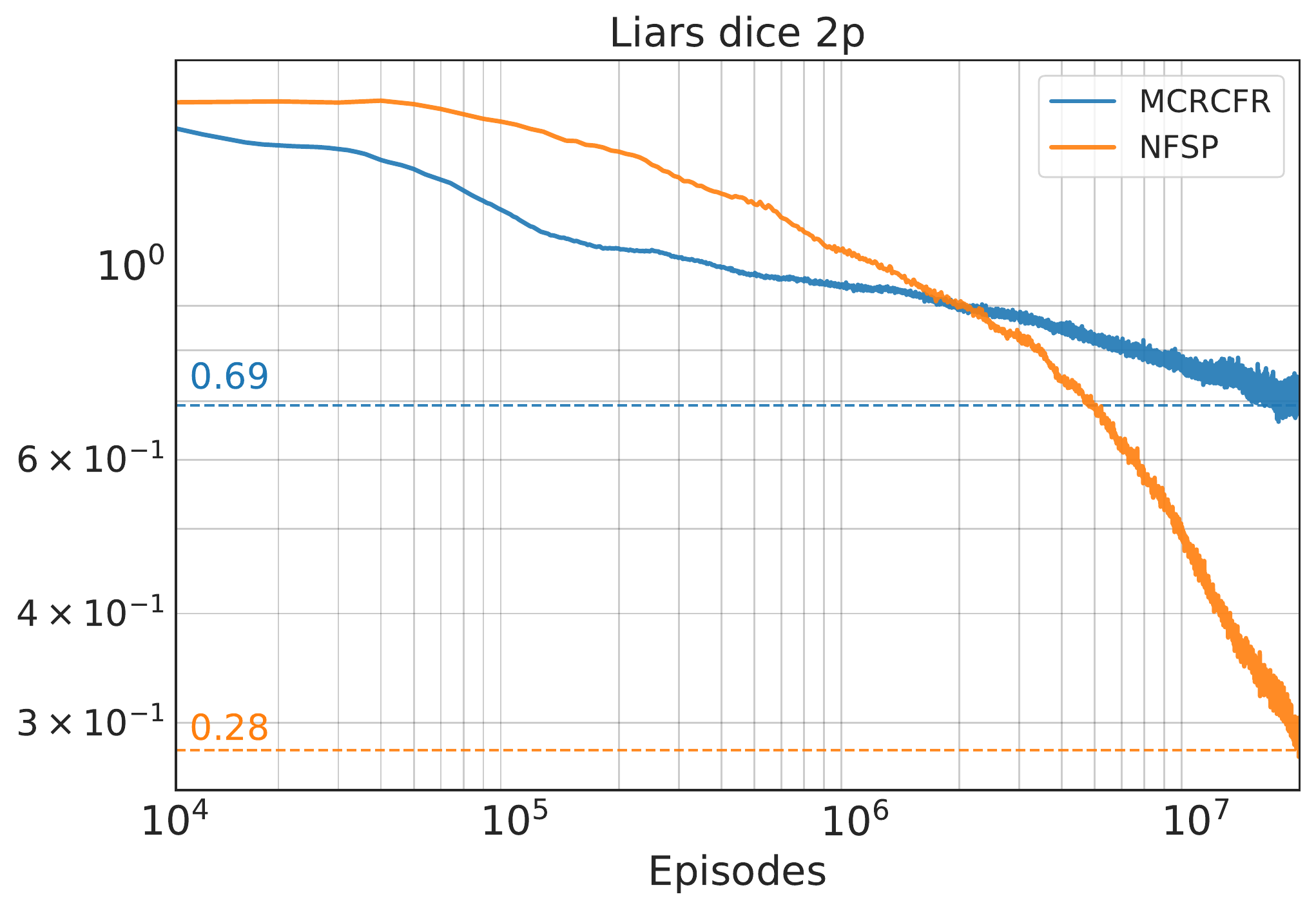}
  \caption{Liars Dice baselines.}
\end{subfigure}
\caption{NFSP and MC-RCFR on the Leduc Poker, II-GoofSpiel with 5 cards and Liars Dice 
}
\label{fig:BI-multiCN}
\end{figure}

\begin{figure}[thbp]
\centering
\begin{subfigure}{0.32\textwidth}
  \centering
  \includegraphics[width=1.0\linewidth]{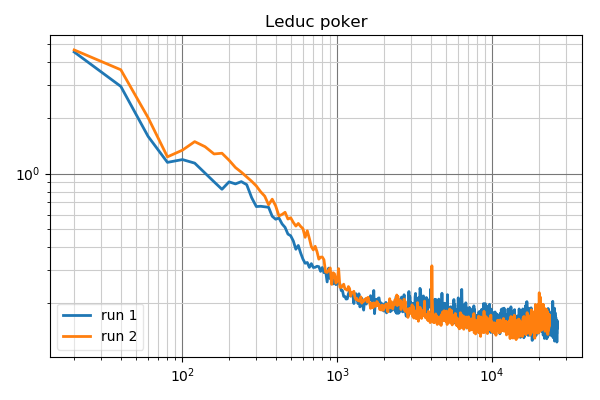}
  \caption{NashConv on Leduc Poker}
\end{subfigure}%
\begin{subfigure}{0.32\textwidth}
  \centering
  \includegraphics[width=1.0\linewidth]{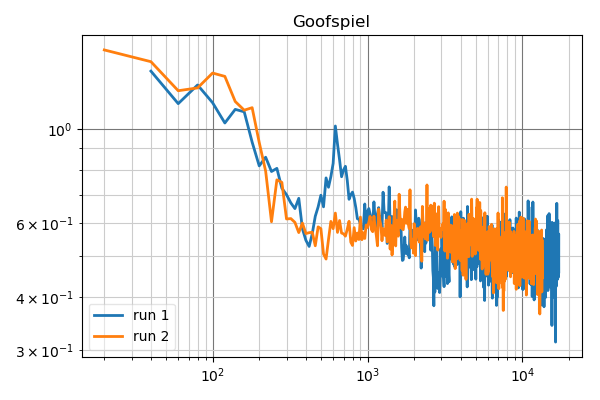}
  \caption{NashConv on Goofspiel.}
\end{subfigure}%
\begin{subfigure}{0.32\textwidth}
  \centering
  \includegraphics[width=1.0\linewidth]{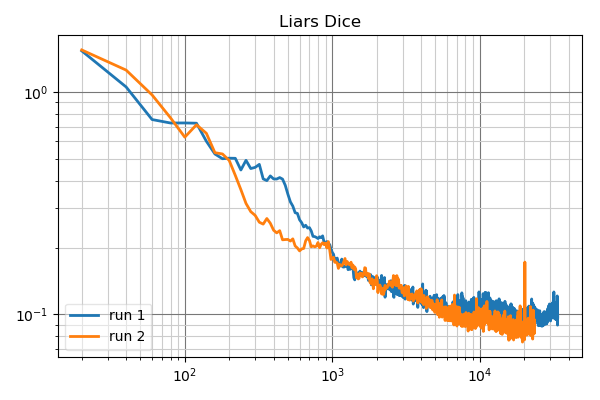}
  \caption{NashConv on Liars Dice}
\end{subfigure}
\caption{\algname results on Leduc, II-Goofspiel, and Liar's Dice. The y-axis is NashConv of the average strategy $\bar{\pi}^t$. The x-axis is number of epochs. One epoch consists of 100 learning steps. Each learning step processes 64 trajectories of length 32 sampled from replay memory. The final value reached by the best runs are 0.18 (Leduc), 0.5 (II-Goofspiel), and 0.095 (Liar's Dice).}
\label{fig:ARMAC-exploit}
\end{figure}

We compare empirical convergence to approximate Nash equilibria using a model-free sampled form of
regression CFR~\cite{Waugh15solving} (MC-RCFR). Trajectories are obtained
using outcome sampling MCCFR~\cite{Lanctot09mccfr},
which uses off-policy importance sampling to obtain unbiased estimates of immediate regrets $\hat{r}$,
and average strategy updates $\hat{s}$, and individual (learned) state-action baselines~\cite{Schmid19VRMCCFR}
to reduce variance.
A regressor then predicts $\hat{\bar{R}}$ and a policy is obtained via Eq.~\ref{eq:regret-matching},
and similarly for the average strategy.
Each episode, the learning player $i$ plays with an $\epsilon$-on-policy behavior policy (while opponent(s) play
on-policy) and adds every datum
$(s, \hat{r}, \pi(\hat{\bar{R}}))$ to a data set, $\cD$, with a retention rule based on reservoir sampling so it
approximates a uniform sample of all the data ever seen.
MC-RCFR is related, but not equivalent to, a variant of DeepCFR~\cite{Brown18DeepCFR} based on outcome sampling (OS-DeepCFR)~\cite{steinberger2020dream}. Our results differ significantly from the OS-DeepCFR results reported in~\cite{steinberger2020dream}, and we discuss differences in assumptions and experimental setup from previous work in Appendix~\ref{app:baseline-details}.
As with \algname, the input is raw bits with no expert features. 
We use networks with roughly the same number of parameters as the \algname experiments: feed-forward with 4 hidden
layers of 128 units with concatenated ReLU~\cite{Shang16crelu} activations, and train using the Adam optimizer.
We provide details of the sweep over hyper-parameters in
Appendix~\ref{app:baseline-details}.

Next we compare \algname to NFSP~\cite{Heinrich16NFSP}, which combines fictitious play with deep neural network function
approximators. Two data sets, $\cD^{RL}$ and $\cD^{SL}$, store transitions of sampled experience for reinforcement learning and
supervised learning, respectively. $\cD^{RL}$ is a sliding window used to train a best response policy to $\bar{\pi}_{-i}$ via DQN.
$\cD^{SL}$ uses reservoir sampling to train $\bar{\pi}_{i}$, an average over all past best response policies. During play, each
agent mixes between its best response policy and average policy.
This stabilizes learning and enables the average policies to
converge to an approximate Nash equilibrium. Like \algname and MC-RCFR, NFSP does not use any expert features. 

Convergence
plots for MC-RCFR and NFSP are shown in Figure~\ref{fig:BI-multiCN}, and for \algname in Figure \ref{fig:ARMAC-exploit}.
NashConv values of \algname are lower (Liar's Dice) and higher (Goofspiel) than NFSP, but significantly lower than MC-RCFR in all cases.
MC-RCFR results are consistent with the outcome sampling results in DNCFR~\cite{Li18DNCFR}. 
Both DNCFR and Deep CFR compensate for this problem by instead using external and robust sampling, which require a forward model.
As with NFSP, and RL more generally, introducing bias to reduce variance may be worthwhile when scaling to large environments.
So, next we investigate the performance of \algname in a much larger game.

\subsection{No-Limit Texas Hold'em}
\label{sec:poker}

\begin{figure}[thbp]
\centering
\begin{subfigure}{1.0\textwidth}
  \centering
  \includegraphics[width=1.0\linewidth]{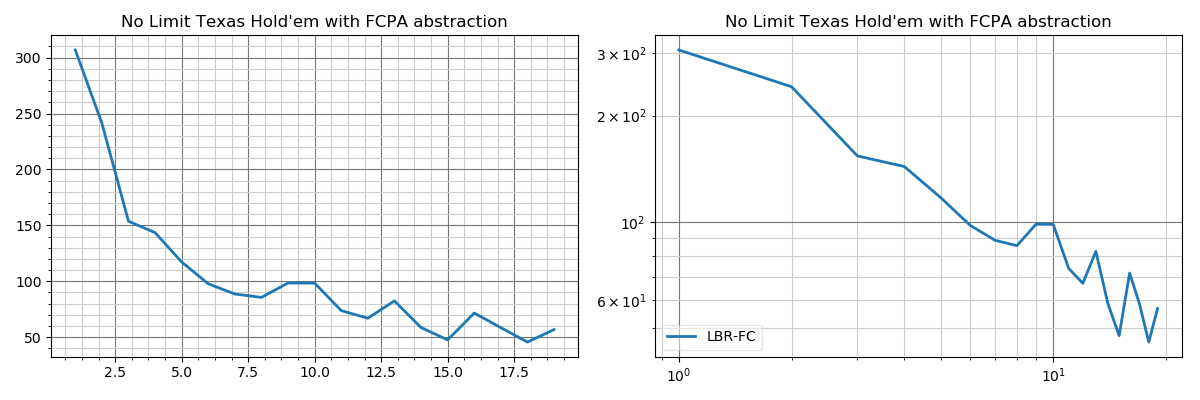}
\end{subfigure}%
\caption{\algname results in No-Limit Texas Hold'em trained with FCPA action abstraction evaluated using LBR-FC metric. The y-axis represents the amount LBR-FC wins agains the ARMAC-trained policy.
The x-axis indicate days of training. The left graph shows the learning curve in a linear scale, while the right one shows the same curve in a log-log scale.} 
\label{fig:ARMAC-LBR}
\end{figure}

We ran \algname on the game of no-limit Texas Hold'em poker, using the common \{ Fold, Call, Pot, All-in \} (FCPA) action/betting
abstraction. This game is orders of magnitude larger than benchmark games used above ($\approx 4.42 \cdot 10^{13}$ information states).
Action abstraction techniques were used by all of the state-of-the-art Poker AI bots up to $2017$.
Modern search-based techniques of DeepStack~\cite{Moravcik17DeepStack} and Libratus~\cite{Brown17Libratus} still include action abstraction, albeit in the search tree.

Computing the NashConv requires traversing the whole game and querying the network at each information state, 
which becomes computationally infeasible as the game grows. So, to assess the quality of the resulting policy, we use local best-response (LBR)~\cite{LisyLBR}. LBR is an exploiter agent that produces a lower-bound on the exploitability: given some policy $\pi_{-i}$ it does a shallow search using the policy at opponent nodes, and a poker-specific heuristic evaluation at the frontier of the search. LBR found that previous competition-winning abstraction-based Poker bots  were far more exploitable than first expected. 
In our experiments, LBR was limited to the betting abstractions: FCPA, and FC. We used three versions of LBR: LBR-FCPA, which uses all 4 actions within the abstraction, LBR-FC, which uses a more limited action set of \{ Fold, Call \} and LBR-FC12-FCPA34 which has a limited action set of \{ Fold, Call \} for the first two rounds and FCPA for the rest.

We first computed the average return that an \algname-trained policy achieves against uniform random.
Over 200000 episodes, the mean value was 516 (chips) $\pm$ 25 (95\% c.i.).
Similarly, we evaluated the policy against LBR-FCPA;
it won 519 $\pm$ 81 (95\% c.i.) per episode. 
Hence, LBR-FCPA was unable to exploit the policy. ARMAC also beat LBR-FC12-FCPA34 by 867 $\pm$ 87 (95\% c.i.) . Interestingly, ARMAC learned to beat those two versions of LBR surprisingly quickly. A randomly initialized ARMAC network lost against LBR-FCPA by -704 $\pm$ 191 (95\% c.i.) and against LBR-FC12-FCPA34 by -230 $\pm$ 222 (95\% c.i.), but was beating both after a mere 1 hour of training by 561 $\pm$ 163 (95\% c.i.) and 427 $\pm$ 140 (95\% c.i.) respectively (~3 million acting steps, ~11 thousand learning steps). \footnote{We verified both implementations LBR against random uniform opponent to make sure that the number matched the ones in \cite{LisyLBR}.}

However, counter-intuitively, ARMAC was exploited by LBR-FC which uses a more limited action set. ARMAC scored -46 $\pm$ 26 (95\% c.i.) per episode  after 18 days of training on a single GPU, 1.3 billion acting steps (rounds), 5 million learning steps, 50000 CFR epochs. The trend over training time is shown in Figure~\ref{fig:ARMAC-LBR}.
To the best of our knowledge, this is the first time a bound on exploitability has been reported for any form of no-limit Texas Hold'em among this class of algorithms.

\section{Conclusion and Future Work}

\algname was demonstrated to work on both single agent and multi-agent benchmarks. It is brings back ideas from computational game theory to address exploration issues while at the same time being able to handle learning in non-stationary environments.
As future work, we intend to apply it to more general classes of multiagent games; \algname has the appealing property that it already stores the joint policies and history-based critics, which may be sufficient for convergence one of the classes of extensive-form correlated equilibria~\cite{celli2019learning,farina2019coarse,celli2020noregret}.

\bibliographystyle{plain}
\bibliography{paper}

\newpage
\appendix
{\Large {\bf Appendices}}

\section{Baseline Details and Hyperparameters}
\label{app:baseline-details}

For MC-RCFR, we sweep over all combinations of the
exploration parameter,
using a (learned) state-action baseline~\cite{Schmid19VRMCCFR},
and learning rate
$(\epsilon, b, \alpha) \in \{ 0.25, 0.5, 0.6, 0.7, 0.9, 0.95, 1.0 \} \times \{ \textsc{True}, \textsc{False} \} \times \{ 0.0001, 0.00005, 0.00001 \}$, where each combination is averaged over five seeds.
We found that higher exploration values worked consistently better, which matches the motivation of the robust
sampling technique (corresponding to $\epsilon = 1$) presented in~\cite{Li18DNCFR} as it leads to reduced variance
since part of the correction term is constant for all histories in an information state.
The baseline helped significantly in the larger game with more variable-length episodes.

For NFSP, we keep a set of hyperparameters fixed, in line with \cite{Lanctot17PSRO} and  \cite{Heinrich16NFSP}: anticipatory parameter $\eta = 0.1$, $\epsilon$-greedy decay duration $20M$ steps, reservoir buffer capacity $2M$ entries, replay buffer capacity $200k$ entries, while sweeping over a combination of the following hyperparameters: $\epsilon$-greedy starting value $\{0.06, 0.24\}$, RL learning rate ${0.1, 0.01, 0.001}$, SL learning rate $\{0.01, 0.001, 0.005\}$,  DQN target network update period of $\{1000, 19200\}$ steps (the later is equivalent to 300 network-parameter updates). Each combination was averaged over three seeds. Agents were trained with the ADAM optimizer, using MSE loss for DQN and one gradient update step using mini-batch size $128$, every $64$ steps in the game.

Finally, note that there are at least four difference in the results, experimental setup, and assumptions between MC-RCFR and OS-DeepCFR reported in~\cite{steinberger2020dream}:
\begin{enumerate}
\item \cite{steinberger2020dream} uses domain expert input features which do not generalize outside of poker. The neural network architecture we use is a basic MLP with raw input representations, whereas \cite{steinberger2020dream} uses a far larger network. Our empirical results on benchmark games compare the convergence properties of knowledge-free algorithms across domains.
\item The amount of training per iteration is an order of magnitude larger in OS-DeepCFR than our training. In \cite{steinberger2020dream}, every 346 iterations, the Q-network is trained using 1000 minibatches of 512 samples (512000 examples), whereas every 346 iterations we train 346 batches of 128 samples, 44288 examples.
\item MC-RCFR uses standard outcome sampling rather than Linear CFR~\cite{BrownLinearCFR}.
\item MC-RCFR's strategy is approximated by predicting the OS's average strategy increment rather than sampling from a buffer of previous models.
\end{enumerate}
Our NFSP also does not use any extra enhancements.

\section{Initial Investigation of \algname in the Atari Learning Environment}
\label{sec:single-agent-atari}

While performance on Atari is not the main contribution, it should be treated as a health check of the algorithm. Unlike previously tested multiplayer games, many Atari games have a long term credit assignment problem. Some of them, like Montezuma's Revenge, are well-known hard exploration problems. It is interesting to see that \algname was able to consistently score 2500 points on Montezuma's Revenge despite not using any auxiliary rewards, demonstrations, or distributional RL as critic. We hypothesize that regret matching may be advantageous for exploration, as it provides naturally stochastic policies which stay stochastic until regrets for other actions becomes negative. We also tested the algorithm on Breakout, as it is a fine control problem. We are not claiming that out results on Atari are state of art - they should be interpreted as a basic sanity check showing that \algname could in principle work in this domain.

\begin{figure}[htbp]
\centering
\begin{subfigure}{0.5\textwidth}
  \centering
  \includegraphics[width=1.0\linewidth]{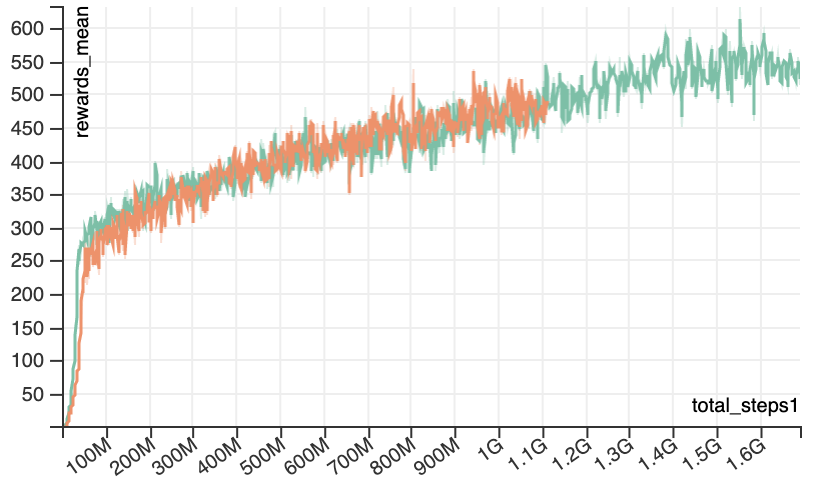}
  \caption{Breakout}
\end{subfigure}%
\begin{subfigure}{0.5\textwidth}
  \centering
  \includegraphics[width=1.0\linewidth]{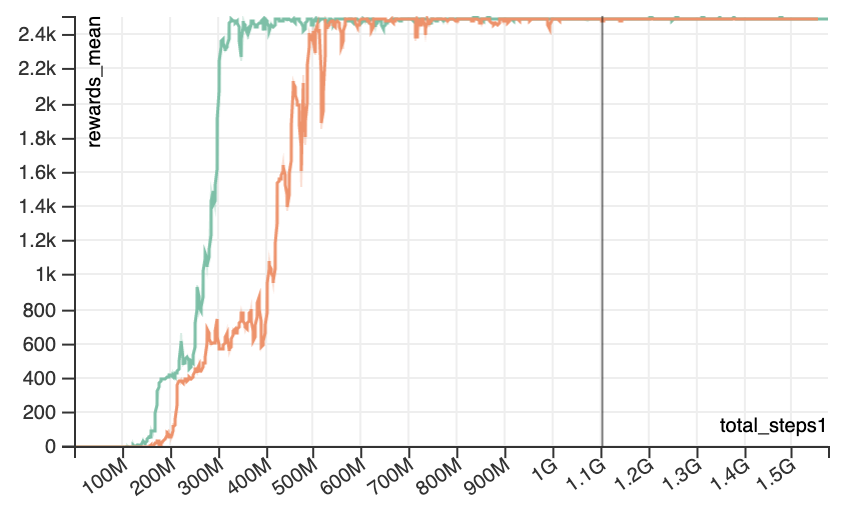}
  \caption{Montezuma Revenge}
\end{subfigure}%
\caption{Performance on Breakout (left) and Montezuma Revenge (right). Results are shown for two seeds.}
\label{fig:atari}
\end{figure}

\section{Training}
\label{sec:training}

Training is done by processing a batch of 64 of trajectories of length 32 at a time. In order to implement a full recall, all unfinished episodes will be continued on the next training iteration by propagating recurrent network states forward. Each time when one episode finishes at a particular batch entry, a new one is sampled and started to be unrolled from the beginning.

Adam optimized with $\beta_1=0.0$ and $\beta_2=0.999$ was used for optimization. Hyperparameter selection was done by trying only two learning rates: $5 \cdot 10^{-5}$ and $2 \cdot 10^{-4}$. The results reported use $5 \cdot 10^{-5}$ in all games, including Atari.

\section{Proof of Lemma \ref{lemma:w}}
\label{sec:proof}

\begin{proof}
First, let us notice that 
\begin{eqnarray}
W_i^T(s,a)&=&  \sum_{h\in s}\frac{\sum_{t=1}^T \eta^{\pi^t}(h)}{\sum_{t=1}^T \eta^{\pi^t}(s)} A_{\pi^t, i}(h,a),\\
&=&  \sum_{h\in s}\frac{\sum_{t=1}^T \eta^{\pi^t}_{-i}(h)}{\sum_{t=1}^T \eta^{\pi^t}_{-i}(s)} A_{\pi^t, i}(h,a)\\
&=& \frac{1}{w^T(s)} \sum_{t=1}^T \sum_{h\in s} \eta^{\pi^t}_{-i}(h) A_{\pi^t, i}(h,a),
\end{eqnarray}
where we used the perfect recall assumption in the first derivation, and we define $w^T(s) =\sum_t\eta^{\pi^t}_{-i}(s)$. Notice that $w^T(s)$ depends on the state only (and not on $h$).
Now the cumulative regret is:
\begin{eqnarray*}
R_i^{T}(s,a)&=&\sum_{t=1}^K q_{\pi^t,i}^c(s, a) - v_{\pi^t,i}^c(s)\\
&=& \sum_{t=1}^T \eta^{\pi^t}_{-i}(s) \big( q_{\pi^t,i}(s, a) - v_{\pi^k,i}(s) \big)\\
&=& \sum_{t=1}^T \eta^{\pi^t}_{-i}(s) \sum_{h\in s} \frac{ \eta^{\pi^t}_{-i}(h)}{ \eta^{\pi^t}_{-i}(s)} \big( q_{\pi^t,i}(h, a) - v_{\pi^t,i}(h) \big)\\
&=& \sum_{t=1}^T
\sum_{h\in s} \eta^{\pi^t}_{-i}(h) A_{\pi^t,i}(h,a)\\
&=& w^T(s) W_i^T(s,a).
\end{eqnarray*}
Finally, noticing that regret matching is not impacted by multiplying the cumulative regret by a positive function of the state, we deduce \[
\frac{R_i^{T,+}(s,a)}{\sum_b R_i^{T,+}(s,b)} = \frac{\big(w^T(s)W_i^{T}(s,a)\big)^+}{\sum_b \big( w^T(s)W_i^{T}(s,b)\big)^+} =  \frac{W_i^{T,+}(s,a)}{\sum_b W_i^{T,+}(s,b)}.
\]
\end{proof}

\section{Unbiasedness of $\hat W^T_i(s,a)$}
\begin{lemma}\label{lemma:unbiased}
Consider the case of a tabular representation and define the estimate $\hat W^T_i(s,a)$ as the minimizer (over $W$) of the empirical loss defined over $N$ trajectories  
$$\hat {\cal L}_{(s,a)}(W)=\frac 1N\sum_{n=1}^N \big[ W - A_{\pi^{j_n},i}(h,a)\big]^2\1\{(h,a)\in \rho^{j_n} \mbox{ and } h\in s\},$$
where $\rho^{j_n}$ is the $n$-th trajectory generated by the policy $(\mu_i^T,\pi_{-i}^{j_n})$ where $j_n\sim {\cal U}(\{1,\dots, T\})$. Define $N(s,a) = \sum_{n=1}^N \1\{(h,a)\in \rho^{j_n}\mbox{ and } h\in s\}$ to be the number of trajectories going through $(s,a)$. Then $\hat W^T_i(s,a)$ is an unbiased estimate of $W^T_i(s,a)$ conditioned on $(s,a)$ being traversed at least once:
$$\E\big[ \hat W^T_i(s,a) | N(s,a)>0\big] = W^T_i(s,a).$$
\end{lemma}

\begin{proof}
The empirical loss being quadratic, under the event $\{N(s,a)>0\}$, its minimum is well defined and reached for 
$$\hat W^T_i(s,a)=\frac {1}{N(s,a)}\sum_{n=1}^{N(s,a)} A_{\pi^{j_n},i}(h_n,a),$$
where $h_n\in s$ is the history of the $n$-th trajectory traversing $s$.
Let us use simplified notations and write $A_n=A_{\pi^{j_n},i}(h,a)\1\{(h,a)\in \rho^{j_n}\mbox{ and } h\in s\}$ and $b_n = \1\{(h,a)\in \rho^{j_n} \mbox{ and } h\in s\}.$ Thus
\begin{eqnarray*}
\E\Big[ \hat W^T_i(s,a) \1\Big\{ \sum_{m=1}^N b_m>0 \Big\} \Big] &=&\E\left[ \frac {\sum_{n=1}^N A_n \1\big\{ \sum_{m=1}^N b_m>0 \big\}}{ \sum_{m=1}^N b_m}\right]\\
&=&\sum_{n=1}^N \E\left[ \E\left[ \frac { A_n \1\big\{ \sum_{m=1}^N b_m>0 \big\}}{\sum_{m=1}^N b_m}\Big| \sum_{m=1}^N b_m\right] \right]\\
&=&\sum_{n=1}^N \E\left[ \E\Big[ A_n \Big| \sum_{m=1}^N b_m\Big]\frac {  \1\big\{ \sum_{m=1}^N b_m>0 \big\}}{\sum_{m=1}^N b_m} \right]\\.
\end{eqnarray*}
Now, $\E\big[ A_n \big| \sum_{m=1}^N b_m\big]  = \E\big[ A_n | b_n \big] \E\big[ b_n | \sum_{m=1}^N b_m\big]$ since given $b_n$, $A_n$ is independent of $\sum_{m=1}^N b_m$. Thus
\begin{eqnarray*}
\E\Big[ \hat W^T_i(s,a)\1\Big\{ \sum_{m=1}^N b_m>0 \Big\}\Big] &=& \sum_{n=1}^N \E\big[ A_n | b_n \big] \E\left[ \E\Big[ \frac { \E\big[b_n\big| \sum_{m=1}^N b_m\big] \1\big\{ \sum_{m=1}^N b_m>0 \big\}}{\sum_{m=1}^N b_m} \right]\\
&=& \sum_{n=1}^N \E\big[ A_n | b_n \big] \E\left[ \frac { b_n \1\big\{ \sum_{m=1}^N b_m>0 \big\}}{\sum_{m=1}^N b_m} \right]
\end{eqnarray*}
Since $\sum_{n=1}^N \E\Big[ \frac {b_n \1\big\{ \sum_{m=1}^N b_m>0 \big\}}{\sum_{m=1}^N b_m}\Big] = \E\Big[ \frac{\sum_{n=1}^N b_n}{\sum_{m=1}^N b_m}\1\big\{ \sum_{m=1}^N b_m>0 \big\}\Big]={\mathbb P}\big(\sum_{m=1}^N b_m>0 \big)$, by a symmetry argument we deduce $\E\Big[ \frac {b_n \1\big\{ \sum_{m=1}^N b_m>0 \big\}}{\sum_{m=1}^N b_m}\Big] = \frac 1N {\mathbb P}\big(\sum_{m=1}^N b_m>0 \big)$ for each $n$. Thus
\begin{eqnarray*}
\E\Big[\hat W^T_i(s,a) \Big| N(s,a)>0 \Big]&=& 
\E\Big[\hat W^T_i(s,a) \Big| \sum_{m=1}^N b_m>0 \Big] \\
&=&\frac{\E\Big[\hat W^T_i(s,a) \1\Big\{ \sum_{m=1}^N b_m>0 \Big\}\Big]}{{\mathbb P}\big(\sum_{m=1}^N b_m>0 \big)}\\
&=&\frac 1N \sum_{n=1}^N \E[A_n|b_n] = \E[A_1|b_1]
\end{eqnarray*}
which is the expectation of the advantage $A_{\pi^j,i}(h,a)$ conditioned on the trajectory $\rho^j$ going through $h\in s$, i.e.~$W^T_i(s,a)$ as defined in~\eqref{eq:W.def}. 
\end{proof}


\section{Neural Network Architecture}
\label{sec:net_arch_2}

The following recurrent neural network was used for no-limit Texas Hold'em experiments. Two separate recurrent networks with shared parameters were used, consuming observations of each player respectively. Each of those networks consisted of a single linear layer mapping input representation to a vector of size 256. This was followed by a double rectified linear unit, producing a representation of size 512 then followed by LSTM with 256 hidden units. This produced an information state representation for each player $a_{0}$ and $a_{1}$.

Define architecture $B(x)$, which will be reused several times. It consumes one of the information state representations produced by the previously mentioned RNN: $h_1 = Linear(128)(x)$, $h_2 = DoubleReLU(h_1)$, $h_3 = h1 + Linear(128)(h2)$, $B(a) = DoubleReLU(h_3)$.

The immediate regret head is formed by applying $B(s)$ on the information state representation followed by a single linear layer of the size of the number of actions in the game. The same is done for an average regret head and mean policy head. All those $B(s)$ do not share weights between themselves, but share weights with respective heads for another player.

The global critic $q(h)$ is defined in the following way.   $n_A = Linear(128)$, $n_B = Linear(128)$, $a_0 = n_A(s_0) + n_B(s_1)$, $a_0 = n_B(s_0) + n_A(s_1)$, $h_1=Concat(a_0, a_1)$, $h_2 = B(h_1)$ and finally $q_0(s_1, s_2)$ and $q_1(s_1, s_2)$ are evaluated by a two linear layers on top of $h_2$. $B(x)$ shares architecture but does not share parameters with the ones used previously.

\end{document}